\documentclass{article}
\usepackage{fullpage}
\usepackage{graphicx} 
\usepackage{subfigure} 

\usepackage[numbers,sort&compress]{natbib}

\usepackage{algorithm}
\usepackage{algorithmic}

\usepackage{hyperref}

\usepackage{mathtools}
\usepackage{amsthm}
\usepackage{amssymb}
\usepackage{amsopn}
\usepackage{mathtools}

\DeclarePairedDelimiter\floor{\lfloor}{\rfloor}
\usepackage{amsthm}
\clearpage{}

\def\D{{\mathcal D}}

\def\reals{{\mathbb R}}

\def\norm#1{\mathopen\| #1 \mathclose\|}

\newcommand{\ignore}[1]{}
\newcommand{\equaldef}{\stackrel{\text{\tiny def}}{=}}

\def\reals{{\mathbb R}}

\def\bold0{\mathbf{0}}

\newcommand{\K}{\ensuremath{\mathcal K}}

\def\eps{\varepsilon}
\def\epsilon{\varepsilon}

\newcommand{\braces}[1]{\left\{#1\right\}}
\newcommand{\pa}[1]{\left(#1\right)}
\newcommand{\ang}[1]{\left<#1\right>}
\newcommand{\bra}[1]{\left[#1\right]}

\DeclareMathOperator*{\E}{\mathbb{E}}

\newcommand{\R}{\mathbb{R}}

\newtheorem{theorem}{Theorem}[section]
\newtheorem{corollary}[theorem]{Corollary}
\newtheorem{definition}[theorem]{Definition}
\newtheorem{lemma}[theorem]{Lemma}
\newtheorem{proposition}[theorem]{Proposition}

\clearpage{}

\newcommand{\punt}[1]{}
\renewcommand{\K}{\mathcal{K}}
\newcommand{\Acal}{\mathcal{A}}
\newcommand{\Lmin}{\lambda_{\mathrm{min}}}
\newcommand{\regret}{\mathfrak{R}}
\newtheorem{assumption}[theorem]{Assumption}

\title{Efficient Regret Minimization in Non-Convex Games}
\author{Elad Hazan\footnote{ehazan@cs.princeton.edu, Department of Computer Science, Princeton University} \and Karan Singh\footnote{karans@cs.princeton.edu, Department of Computer Science, Princeton University} \and Cyril Zhang\footnote{cyril.zhang@princeton.edu, Department of Computer Science, Princeton University}}

\begin{document} 

\maketitle

\begin{abstract}
We consider regret minimization in repeated games with non-convex loss functions. Minimizing the standard notion of regret is computationally intractable. Thus, we define a natural notion of regret which permits efficient optimization and generalizes offline guarantees for convergence to an approximate local optimum. We give gradient-based methods that achieve optimal regret, which in turn guarantee convergence to equilibrium in this framework.
\end{abstract}

\section{Introduction}

Repeated games with non-convex utility functions serve to model many natural settings, such as multiplayer games with risk-averse players and adversarial (e.g. GAN) training. However, standard regret minimization and equilibria computation with general non-convex losses are computationally hard. 
This paper studies \emph{computationally tractable} notions of regret minimization and equilibria in non-convex repeated games. 

Regret minimization in games typically amounts to repeated play in which the decision maker accumulates an average loss proportional to that of the best fixed decision in hindsight. This is a {\it global} notion with respect to the decision set of the player. If the loss functions are convex (or, as often considered, linear) restricted to the actions of the other players, then this notion of global optimization is computationally tractable. It can be shown that under certain conditions, players that minimize regret converge in various notions to standard notions of equilibrium, such as Nash equilibrium, correlated equilibrium, and coarse correlated equilibrium. This convergence crucially relies on the global optimality guaranteed by regret.

In contrast, it is NP-hard to compute the global minimum of a non-convex function over a convex domain. Rather, efficient non-convex continuous optimization algorithms focus on finding a {\it local} minimum. We thus consider notions of equilibrium that can be obtained from local optimality conditions of the players with respect to each-others' strategies.  
This requires a different notion of regret whose minimization guarantees convergence to a local minimum.

The rest of the paper is organized as follows.
After briefly discussing why standard regret is not a suitable metric of performance, we introduce and motivate \textit{local regret}, a surrogate for regret to the non-convex world. We then proceed to give efficient algorithms for non-convex online learning with optimal guarantees for this new objective. In analogy with the convex setting, we discuss the way our framework captures the offline and stochastic cases. In the final section, we describe a game-theoretic solution concept which is intuitively appealing, and, in contrast to other equilibria, efficiently attainable in the non-convex setting by simple algorithms.

\subsection{Related work}

The field of online learning is by now rich with a diverse set of algorithms for extremely general scenarios, see e.g. \citep{CesaBianchiLugosi06book}. For bounded cost functions over a bounded domain, it is well known that versions of the multiplicative weights method gives near-optimal regret bounds \citep{cover,Vovk:1990,AHK-MW}. 

Despite the tremendous generality in terms of prediction, the multiplicative weights method in its various forms yields only exponential-time algorithms for these general scenarios. This is inevitable, since regret minimization implies optimization, and general non-convex optimization is NP-hard. Convex forms of regret minimization have dominated the learning literature in recent years due to the fact that they allow for efficient optimization, see e.g. \cite{OCObook,shalev2011online}. 

Non-convex mathematical optimization algorithms typically find a local optimum. For smooth optimization, gradient-based methods are known to find a point with gradient of squared norm at most $\varepsilon$ in $O(\frac{1}{\varepsilon})$ iterations \citep{NesterovBook}.\footnote{We note here that we measure the squared norm of the gradient, since it is more compatible with convex optimization. The mathematical optimization literature sometimes measures the norm of the gradient without squaring it.} A rate of $O(\frac{1}{\varepsilon^2})$ is known for stochastic gradient descent \citep{ghadimi-lan}. Further accelerations in terms of the dimension are possible via adaptive regularization \citep{adagrad}.

Recently, stochastic second-order methods have been considered, which enable even better guarantees for non-convex optimization: not only is the gradient at the point returned small, but the Hessian is also guaranteed to be close to positive semidefinite (i.e. the objective function is locally almost-convex), see e.g. \cite{newsamp,CarmonAGD,Lissa2,LiSSA2016}.

The relationship between regret minimization and learning in games has been considered in both the machine learning literature, starting with \cite{FreundSch1997}, and the game theory literature by \cite{hart2000simple}. Motivated by \cite{hart2000simple}, \cite{blumMansour} study reductions from internal to external regret, and \cite{NIPS2007_695} relate the computational efficiency of these reductions to fixed point computations.

\section{Setting}

We begin by introducing the setting of online non-convex optimization, which is modeled as a game between a learner and an adversary. During each iteration $t$, the learner is tasked with predicting $x_t$ from $\K \subseteq \reals^n$, a convex decision set. Concurrently, the adversary chooses a loss function $f_t: \K\to\reals$; the learner then observes $f_t(x)$ (via access to a first-order oracle) and suffers a loss of $f_t(x_t)$. This procedure of play is repeated across $T$ rounds.

The performance of the learner is measured through its regret, which is defined as a function of the loss sequence $f_1, \ldots, f_T$ and the sequence of online decisions $x_1, \ldots, x_T$ made by the learner. We discuss our choice of regret measure at length in Section~\ref{subsection:regret-measure}.

Throughout this paper, we assume the following standard regularity conditions:
\begin{assumption}\label{assumption:smooth}
We assume the following is true for each loss function $f_t$:
\begin{enumerate}
\item[(i)] $f_t$ is bounded:
$|f_t(x)| \leq M.$
\item[(ii)] $f_t$ is $L$-Lipschitz:
$|f_t(x)-f_t(y)| \leq L\|x-y\|.$
\item[(iii)] $f_t$ is $\beta$-smooth (has a $\beta$-Lipschitz gradient):
\[\norm{\nabla f_t(x) - \nabla f_t(y)} \leq \beta \norm{x - y}.\]
\end{enumerate}
\end{assumption}

\subsection{Projected gradients and constrained non-convex optimization}

In constrained non-convex optimization, minimizing the gradient presents difficult computational challenges. In general, even when objective functions are smooth and bounded, local information may provide no information about the location of a stationary point. This motivates us to refine our search criteria.

Consider, for example, the function sketched in Figure~\ref{fig:needle}. In this construction, defined on the hypercube in $\R^n$, the unique point with a vanishing gradient is a hidden valley, and gradients outside this valley are all identical. Clearly, it is hopeless in an information-theoretic sense to find this point efficiently: the number of value or gradient evaluations of this function must be $\exp(\Omega(n))$ to discover the valley.

\begin{center}
\begin{figure}[h!]
\begin{center}
\caption{A difficult ``needle in a haystack'' case for constrained non-convex optimization. \emph{Left:} A function with a hidden valley, with small gradients shown in yellow. \emph{Right:} Regions with small \emph{projected} gradient for the same function. For smaller $\eta$, only points near the valley and bottom-left corner have small projected gradient.} \label{fig:needle}
\vspace{4mm}
\includegraphics[width=0.15\textwidth]{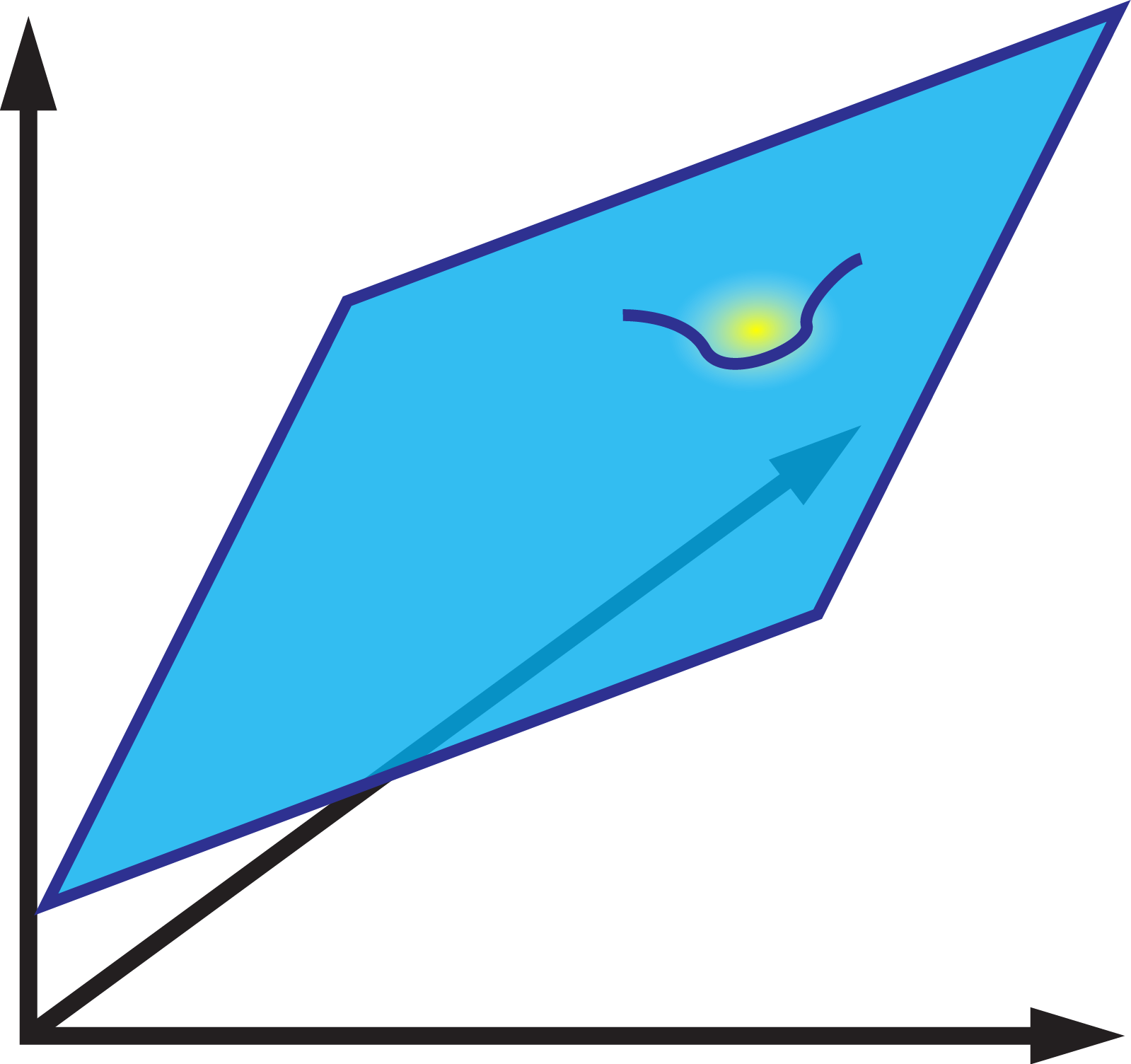}
\hspace{12mm}
\includegraphics[width=0.15\textwidth]{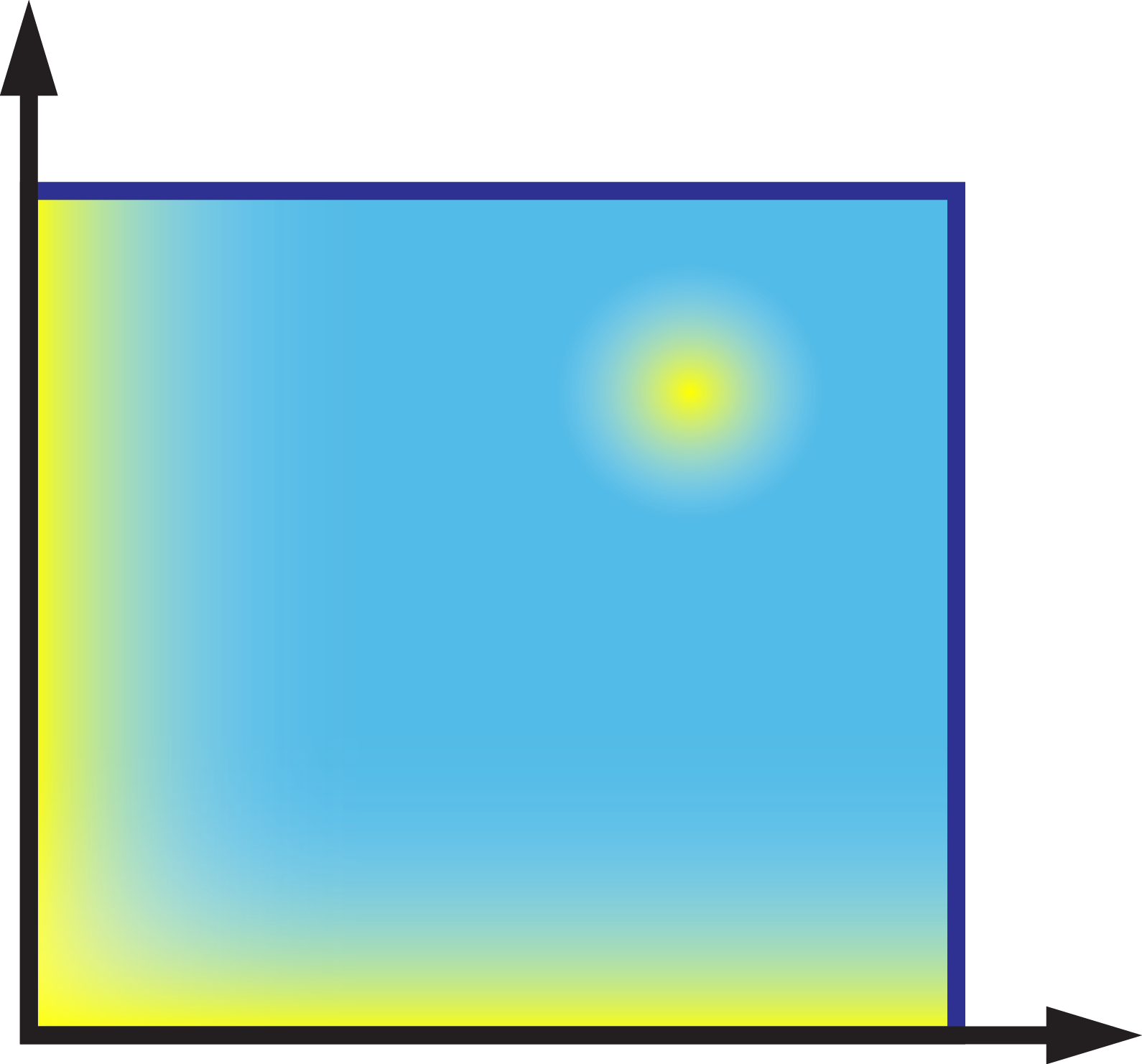}
\end{center}
\end{figure}
\end{center}

To circumvent such inherently difficult and degenerate cases, we relax our conditions, and try to find a vanishing \emph{projected gradient}. In this section, we introduce this notion formally, and motivate it as a natural quantity of interest to capture the search for local minima in constrained non-convex optimization.

\begin{definition}[Projected gradient]
Let $f : \K \rightarrow \R$ be a differentiable function on a closed (but not necessarily bounded) convex set $\K \subseteq \R^n$. Let $\eta > 0$.
We define $\nabla_{\K, \eta} f : \K \rightarrow \R^n$, the \emph{$(\K,\eta)$-projected gradient} of $f$, by
\begin{equation*}
\nabla_{\K, \eta} f(x) \equaldef
\frac{1}{\eta} \pa{ x - \Pi_\K\bra{ x - \eta \nabla f(x) } },
\end{equation*}
where $\Pi_\K[\cdot]$ denotes the orthogonal projection onto $\K$.
\end{definition}

This can be viewed as a surrogate for the gradient which ensures that the gradient descent step always lies within $\K$, by transforming it into a projected gradient descent step. Indeed, one can verify by definition that
\[x - \eta \nabla_{\K, \eta}(x) = \Pi_\K\bra{ x - \eta \nabla f(x) }.\]

In particular, when $\K = \reals^n$,
\[\nabla_{\K,\eta} f(x) = \frac{1}{\eta}(x - x + \eta \nabla f(x)) = \nabla f(x),\] 
and we retrieve the usual gradient at all $x$.

We first note that there always exists a point with vanishing projected gradient.

\begin{proposition}
\label{fixed-point}
Let $\K$ be a compact convex set, and suppose $f : \K \rightarrow \R$ satisfies Assumption~\ref{assumption:smooth}. Then, there exists some point $x^* \in \K$ for which
\begin{equation*}
\nabla_{\K,\eta} f(x^*) = 0.
\end{equation*}
\end{proposition}
\begin{proof}
Consider the map $g : \K \rightarrow \K$, defined by
\begin{equation*}
g(x) \equaldef x - \eta \nabla_{\K, \eta} f(x) = \Pi_\K \bra{ x - \eta \nabla f(x) }.
\end{equation*}
This is a composition of continuous functions (noting that the smoothness assumption implies that $\nabla f$ is continuous), and is therefore continuous. Thus $g$ satisfies the conditions for Brouwer's fixed point theorem, implying that there exists some $x^* \in \K$ for which $g(x^*) = x^*$. At this point, the projected gradient vanishes.
\end{proof}

In the limit where $\eta \norm{\nabla f(x)}$ is infinitesimally small, the projected gradient is equal to the gradient in the interior of $\K$; on the boundary of $\K$, it is the gradient with its outward-facing component removed. This exactly captures the first-order condition for a local minimum.

The final property that we note here is that an approximate local minimum, as measured by a small projected gradient, is robust with respect to small perturbations.

\begin{proposition}
\label{projection-is-lipschitz}
Let $x$ be any point in $\K \subseteq \R^n$, and let $f,g$ be differentiable functions $\K \rightarrow \R$. Then, for any $\eta > 0$,
\[ \norm{\nabla_{\K,\eta} [f + g](x)} \leq \norm{\nabla_{\K,\eta} f(x)} + \norm{\nabla g(x)}. \]
\end{proposition}
\begin{proof}
Let $u = x + \eta \nabla f(x)$, and $v = u + \eta \nabla g(x)$. Define their respective projections $u' = \Pi_\K\bra{u}, v' = \Pi_\K\bra{v}$, so that $u' = x - \eta \nabla_{\K,\eta} f(x)$ and $v' = x - \eta \nabla_{\K,\eta} [f + g](x)$. We first show that $\norm{u' - v'} \leq \norm{u - v}$.

By the generalized Pythagorean theorem for convex sets, we have both $\ang{u' - v', v - v'} \leq 0$ and
$\ang{v' - u', u - u'} \leq 0$.
Summing these, we get
\begin{align*}
\ang{u' - v', u' - v' - (u - v)} &\leq 0 \\
\implies \norm{u' - v'}^2 &\leq \ang{u' - v', u - v}\\
&\leq \norm{u'-v'} \cdot \norm{u-v},
\end{align*}
as claimed.
Finally, by the triangle inequality, we have
\begin{align*}
\norm{\nabla_{\K,\eta} [f &+ g](x)} - \norm{\nabla_{\K,\eta} f(x)} \\
&\leq \norm{\nabla_{\K,\eta} [f + g](x) - \nabla_{\K,\eta} f(x)}\\
&= \frac{1}{\eta} \norm{u'-v'} \\
& \leq \frac{1}{\eta} \norm{u-v} = \norm{\nabla g(x)},
\end{align*}
as required.

\end{proof}
In particular, this fact immediately implies that $\norm{ \nabla_{\K, \eta} f(x) } \leq \norm{ \nabla f(x) }$.

As we demonstrate later, looking for a small projected gradient becomes a feasible task. In Figure~\ref{fig:needle} above, such a point exists on the boundary of $\K$, even when there is no ``hidden valley'' at all.

\subsection{A local regret measure}
\label{subsection:regret-measure}

In the well-established framework of online convex optimization, numerous algorithms can efficiently achieve optimal regret, in the sense of converging in terms of average loss towards the best fixed decision in hindsight. That is, for any $u \in \K$, one can play iterates $x_1, \ldots, x_T$ such that
\begin{equation*}
\frac{1}{T} \sum_{i=1}^T \bra{ f_t(x_t) - f_t(u) } = o(1).
\end{equation*}
Unfortunately, even in the \emph{offline} case, it is too ambitious to converge towards a global minimizer in hindsight. In the existing literature, it is usual to state convergence guarantees towards an $\eps$-approximate stationary point -- that is, there exists some iterate $x_t$ for which $\norm{\nabla f(x_t)}^2 \leq \eps$. As discussed in the previous section, the projected gradient is a natural analogue for the constrained case.

In light of the computational intractability of direct analogues of convex regret, we introduce \emph{local regret}, a new notion of regret which quantifies the objective of predicting points with small gradients on average. The remainder of this paper discusses the motivating roles of this quantity.

Throughout this paper, for convenience, we will use the following notation to denote the sliding-window time average of functions $f$, parametrized by some window size $1 \leq w \leq T$:
\begin{equation*}
F_{t,w}(x) \equaldef \frac{1}{w} \sum_{i=0}^{w-1} f_{t-i}(x).
\end{equation*}
For simplicity of notation, we define $f_t(x)$ to be identically zero for all $t \leq 0$. We define local regret below:

\begin{definition}[Local regret]
Fix some $\eta > 0$. Define the \emph{$w$-local regret} of an online algorithm as
\begin{equation*}
\regret_w(T) \equaldef \sum_{t=1}^T \; \norm{ \nabla_{\K, \eta} F_{t,w}(x_t) }^2,
\end{equation*}
\end{definition}

When the window size $w$ is understood by context, we omit the parameter, writing simply \emph{local regret} as well as $F_t(x)$.

We turn to the first motivating perspective on local regret. When an algorithm incurs local regret sublinear in $T$, a randomly selected iterate has a small time-averaged gradient in expectation:
\begin{proposition}
\label{regret-to-convergence}
Let $x_1, \ldots, x_T$ be the iterates produced by an algorithm for online non-convex optimization which incurs a local regret of $\regret_w(T)$. Then,
\begin{equation*}
\E_{t \sim \mathrm{Unif}([T])}\bra{ \norm{ \nabla_{\K, \eta} F_{t,w}(x_t) }^2 } \leq \frac{\regret_w(T)}{T}.
\end{equation*}
\end{proposition}
This generalizes typical convergence results for the gradient in offline non-convex optimization; we discuss concrete reductions in Section~\ref{reduction-section}.

\subsection{Why smoothing is necessary}
In this section, we show that for any online algorithm, an adversarial sequence of loss functions can force the local regret incurred to scale with $T$ as $\Omega\left(\frac{T}{w^2}\right)$. This demonstrates the need for a time-smoothed performance measure in our setting, and justifies our choice of larger values of the window size $w$ in the sections that follow.
\begin{theorem}
Define $\K=[-1,1]$. For any $T \geq 1$, $1 \leq w \leq T$, and $\eta \leq 1$, there exists a distribution $\mathcal{D}$ on $0$-smooth, $1$-bounded cost functions $f_1,\dots, f_T$ on $\K$ such that for any online algorithm, when run on this sequence of functions,
\[\E_\mathcal{D} \bra{ \regret_w(T) } \geq \frac{1}{4w} \left\lfloor \frac{T}{2w} \right\rfloor. \]
\end{theorem}
\begin{proof}
We begin by partitioning the $T$ rounds of play into $\floor{\frac{T}{2w}}$ repeated segments, each of length $2w$.

For the first half of the first segment ($t = 1, \ldots, w$), the adversary declares that
\begin{itemize}
\item For odd $t$, select $f_t(x)$ i.i.d. as follows:
\begin{align*}
f_t(x) := \begin{cases}
-x, & \textrm{with probability }\frac{1}{2}\\
x, & \textrm{with probability }\frac{1}{2}
\end{cases}
\end{align*}
\item For even $t$, $f_t(x) := -f_{t-1}(x)$.
\end{itemize}

During the second half ($t = w+1, \ldots, 2w$), the adversary sets all $f_t(x) = 0$. This construction is repeated $\lfloor \frac{T}{2w} \rfloor$ times, padding the final $T\textrm{ mod }2w$ costs arbitrarily with $f_t(x) = 0$.

By this construction, at each round $t$ at which $f_t(x)$ is drawn randomly, we have
$F_{t,w}(x) = f_t(x) / w$. Furthermore, for any $x_t$ played by the algorithm, $|\nabla_{\K,\eta} f_t(x_t)| = 1$ with probability at least $\frac{1}{2}$. so that $\E\bra{ \norm{ \nabla_{\K, \eta} F_{t,w}(x_t) }^2 } \geq \frac{1}{2w^2}$. The claim now follows from the fact that there are at least $\frac{w}{2}$ of these rounds per segment, and exactly $\left\lfloor \frac{T}{2w} \right\rfloor$ segments in total.
\end{proof}

We further note that the notion of time-smoothing captures non-convex online optimization under limited \emph{concept drift}: in online learning problems where $F_{t,w}(x) \approx f_t(x)$, a bound on local regret truly captures a guarantee of playing points with small gradients. 
\section{An efficient non-convex regret minimization algorithm}

Our approach, as given in Algorithm~\ref{time-smoothed-ogd}, is to play follow-the-leader iterates, approximated to a suitable tolerance using projected gradient descent. We show that this method efficiently achieves an optimal local regret bound of $O\pa{ \frac{T}{w^2} }$, taking $O\pa{Tw}$ iterations of the inner loop.

\begin{algorithm}
\caption{Time-smoothed online gradient descent}
\label{time-smoothed-ogd}
\begin{algorithmic}[1]
\STATE Input: window size $w \geq 1$, learning rate $0 < \eta < \frac{\beta}{2}$, tolerance $\delta > 0$, a convex body $\K\subseteq \reals^n$.
\STATE Set $x_1 \in \K$ arbitrarily.
\FOR{$t = 1, \ldots, T$}
\STATE Predict $x_t$. Observe the cost function $f_t:\K\to \reals$.
\STATE Initialize $x_{t+1} := x_t$.
\WHILE{ $\norm{ \nabla_{\K,\eta} F_{t,w} (x_{t+1}) } > \delta/w$ }
\STATE Update $x_{t+1} := x_{t+1} - \eta \nabla_{\K,\eta} F_{t,w} (x_{t+1})$.
\ENDWHILE
\ENDFOR
\end{algorithmic}
\end{algorithm}

\begin{theorem}
\label{thm:onco-first}
Let $f_1, \ldots, f_T$ be the sequence of loss functions presented to Algorithm~\ref{time-smoothed-ogd}, satisfying Assumption~\ref{assumption:smooth}. Then:
\begin{enumerate}
\item[(i)] The $w$-local regret incurred satisfies \[\regret_w(T) \leq \pa{\delta + 2L}^2 \frac{T}{w^2}.\]
\item[(ii)] The total number of gradient steps $\tau$ taken by Algorithm~\ref{time-smoothed-ogd} satisfies \[\tau \leq \frac{M}{\delta^2 \pa{\eta - \frac{\beta \eta^2}{2}}} \cdot \pa{2Tw + w^2}.\]
\end{enumerate}
\end{theorem}

\begin{proof}[Proof of (i).]
We note that Algorithm~\ref{time-smoothed-ogd} will only play an iterate $x_t$ if $\norm{\nabla_{\K,\eta} F_{t-1,w}} \leq \delta / w$. (Note that at $t=1$, $F_{t-1,w}$ is zero.) Let $h_t(x) = \frac{1}{w}\pa{ f_t(x) - f_{t-w}(x) }$, which is $\frac{2L}{w}$-Lipschitz. Then, for each $1 \leq t \leq T$ we have a bound on each cost
\begin{align*}
\norm{\nabla_{\K,\eta} &F_{t,w}(x_t)}^2
= \norm{\nabla_{\K,\eta} \bra{ F_{t,w-1} + h_t(x)}(x_t)}^2 \\
&\leq \pa{ \norm{\nabla_{\K,\eta} F_{t,w-1}} + \norm{\nabla h_t(x_t)} }^2 \\
&\leq \pa{ \frac{\delta}{w} + \frac{2L}{w} }^2 = \frac{(\delta + 2L)^2}{w^2},
\end{align*}
where the first inequality follows from Proposition~\ref{projection-is-lipschitz}. Summing over all $t$ gives the desired result.
\end{proof}

\begin{proof}[Proof of (ii).]
First, we require an additional property of the projected gradient.
\begin{lemma}
\label{lem:py}
Let $\K \in \R^n$ be a closed convex set, and let $\eta > 0$. Suppose $f:\K\to\reals$ is differentiable. Then, for any $x \in \R$,
\begin{equation*}
\ang{\nabla f(x), \nabla_{\K,\eta} f(x)} \geq \norm{ \nabla_{\K,\eta} f(x) }^2.
\end{equation*}
\end{lemma}
\begin{proof}
Let $u = x - \eta \nabla f(x)$ and $u' = \Pi_\K \bra{u}$. Then,
\begin{align*}
&\frac{\ang{\nabla f(x), \nabla_{\K,\eta} f(x)}}{\eta^2} - \frac{\norm{ \nabla_{\K,\eta} f(x) }^2}{\eta^2}\\
&= \ang{u-x,u'-x} - \ang{u'-x,u'-x} \\
&= \ang{u-u',u'-x} \geq 0,
\end{align*}
where the last inequality follows by the generalized Pythagorean theorem.
\end{proof}

For $2 \leq t \leq T$, let $\tau_t$ be the number of gradient steps taken in the outer loop at iteration $t-1$, in order to compute the iterate $x_t$.
For convenience, define $\tau_1 = 0$. We establish a progress lemma during each gradient descent epoch:

\begin{lemma}
\label{new-alg-progress}
For any $2 \leq t \leq T$,
\[ F_{t-1}(x_t) - F_{t-1}(x_{t-1}) \leq - \tau_t \pa{ \eta - \frac{\beta \eta^2}{2} } \frac{\delta^2}{w^2}. \]
\end{lemma}
\begin{proof}
Consider a single iterate $z$ of the inner loop, and the next iterate $z' := z - \eta \nabla_{\K,\eta} F_{t-1}(z)$. We have, by $\beta$-smoothness of $F_{t-1}$,
\begin{align*}
&F_{t-1}(z') - F_{t-1}(z) \leq \ang{\nabla F_{t-1}(z), z' - z} + \frac{\beta}{2}\norm{z' - z}^2 \\
&= -\eta\ang{\nabla F_{t-1}(z), \nabla_{\K,\eta} F_{t-1} (z)} + \frac{\beta \eta^2}{2}\norm{\nabla_{\K,\eta} F_{t-1} (z)}^2.
\end{align*}
Thus, by Lemma~\ref{lem:py},
\begin{align*}
F_{t-1}(z') - F_{t-1}(z) &\leq -\pa{\eta - \frac{\beta \eta^2}{2}} \norm{\nabla_{\K,\eta} F_{t-1} (z)}^2.
\end{align*}
The algorithm only takes projected gradient steps when $\norm{\nabla_{\K,\eta} F_{t-1} (z)} \geq \delta / w$. 
Summing across all $\tau_t$ consecutive iterations in the epoch yields the claim.
\end{proof}

To complete the proof of the theorem, we write the telescopic sum (understanding $F_{0}(x_0) = 0$):
\begin{align*}
&F_T(x_T) = \sum_{t=1}^{T} F_{t}(x_{t}) - F_{t-1}(x_{t-1}) \\
&= \sum_{t=1}^{T} F_{t-1}(x_t) - F_{t-1}(x_{t-1}) + f_t(x_t) - f_{t-w}(x_t) \\
&\leq \sum_{t=2}^{T} \bra{ F_{t-1}(x_t) - F_{t-1}(x_{t-1}) } + \frac{2MT}{w}.
\end{align*}
Using Lemma~\ref{new-alg-progress}, we have
\begin{align*}
F_T(x_T) &\leq \frac{2MT}{w} - \frac{\pa{\eta - \frac{\beta \eta^2}{2}} \delta^2}{w^2} \cdot \sum_{t=1}^T \tau_t,
\end{align*}
whence
\begin{align*}
\tau = \sum_{t=1}^T \tau_t &\leq \frac{w^2}{\delta^2 \pa{\eta - \frac{\beta \eta^2}{2}}} \cdot \pa{ \frac{2MT}{w} - F_T(x_T) } \\
&\leq \frac{M}{\delta^2 \pa{\eta - \frac{\beta \eta^2}{2}}} \cdot \pa{2Tw + w^2},
\end{align*}
as claimed.
\end{proof}
Setting $\eta = 1/\beta$ and $\delta = L$ gives the asymptotically optimal local regret bound, with $O(Tw)$ time-averaged gradient steps (and thus $O(Tw^2)$ individual gradient oracle calls). We further note that in the case where $\K = \R^n$, one can replace the gradient descent subroutine (the inner loop) with non-convex SVRG \cite{allen2016variance}, achieving a complexity of $O(Tw^{5/3})$ gradient oracle calls.

\section{Implications for offline and stochastic non-convex optimization}
\label{reduction-section}

In this section, we discuss the ways in which our online framework generalizes the offline and stochastic versions of non-convex optimization -- that any algorithm achieving a small value of $\regret_w(T)$ efficiently finds a point with small gradient in these settings. For convenience, for $1 \leq t \leq t' \leq T$, we denote by $\mathcal{D}_{[t,t']}$ the uniform distribution on time steps $t$ through $t'$ inclusive.

\subsection{Offline non-convex optimization}
For offline optimization on a fixed non-convex function $f:\K\to\reals$, we demonstrate that a bound on local regret translates to convergence. In particular, using Algorithm~\ref{time-smoothed-ogd} one finds a point $x\in\K$ with $\|\nabla_{\K,\eta} f(x)\|^2\leq \eps$ while making $O\left(\frac{1}{\eps}\right)$ calls to the gradient oracle, matching the best known result for the convergence of gradient-based methods.

\begin{corollary}
Let $f:\K\to\reals$ satisfy Assumption~\ref{assumption:smooth}. When online algorithm $\Acal$ is run on a sequence of $T$ identical loss functions $f(x)$, it holds that for any $1 \leq w < T$,
\begin{equation*}
\mathbb{E}_{t\sim \mathcal{D}_{[w,T]}} \|\nabla_{\K,\eta} f(x_t)\|^2  \leq  \frac{\regret_w(\Acal)}{T-w}.
\end{equation*}
In particular, Algorithm~\ref{time-smoothed-ogd}, with parameter choices $T=2w, \eta = \frac{1}{\beta}, \delta = L$, and $w=(\delta+2L)\sqrt{\frac{2}{\eps}}$, yields
\[ \mathbb{E}_{t\sim\mathcal{D}_{w,T}}\|\nabla_{\K,\eta} f(x_t)\|^2 \leq \eps.\]
Furthermore, the algorithm makes $O\left(\frac{1}{\eps}\right)$ calls to the gradient oracle in total.
\end{corollary}
\begin{proof}
Since $f_t(x) = f(x)$ for all $t$, it follows that $F_{t,w}(x)=f(x)$ for all $t \geq w$. As a consequence, we have
\begin{align*}
\mathbb{E}_{t\sim \mathcal{D}_{w,t}} \|\nabla_{\K,\eta} f(x_t)\|^2   &\leq \frac{1}{T-w} \sum_{t=1}^T \|\nabla_{\K,\eta} f(x_t)\|^2 \\
&\leq \frac{\regret_w(\Acal)}{T-w}.
\end{align*}

With the stated choice of parameters, Theorem~\ref{thm:onco-first} guarantees that
\[\mathbb{E}_{t\sim \mathcal{D}_{w,t}} \|\nabla_{\K,\eta} f(x_t)\|^2 \leq \frac{\eps}{2} \cdot \frac{T}{T-w} = \eps.\]
Also, since the loss functions are identical, the execution of line 7 of Algorithm~\ref{time-smoothed-ogd} requires exactly one call to the gradient oracle at each iteration. This entails that the total number of gradient oracle calls made in the execution is $O(Tw+w^2) = O(\frac{1}{\eps})$.
\end{proof}

\subsection{Stochastic non-convex optimization}
\label{stochastic-section}
We examine the way in which our online framework captures stochastic non-convex optimization of a fixed function
$f:\reals^n\to\reals$, in which an algorithm has access to a noisy \emph{stochastic gradient} oracle $\widetilde{\nabla f}(x)$. We note that the reduction here will only apply in the unconstrained case; it becomes challenging to reason about the projected gradient under noisy information. From a local regret bound, we recover a stochastic algorithm with oracle complexity $O\pa{ \frac{\sigma^4}{\eps^2} }$. We note that this black-box reduction recovers an optimal convergence rate in terms of $\eps$, but not $\sigma^2$.

In the setting, the algorithm must operate on the noisy estimates of the gradient as the feedback. In particular, for any $f_t$ that the adversary chooses, the learning algorithm is supplied with a stochastic gradient oracle for $f_t$. The discussion in the preceding sections may be viewed as a special case of this setting with $\sigma=0$. We list the assumptions we make on the stochastic gradient oracle, which are standard:
\begin{assumption}
\label{u1}
We assume that each call to the stochastic gradient oracle yields an i.i.d. random vector $\widetilde{\nabla f}(x)$ with the following properties:
\begin{enumerate}
\item[(i)] Unbiased: $\E \bra{ \widetilde{\nabla f}(x) } = \nabla f(x).$
\item[(ii)] Bounded variance: $\E\bra{\norm{ \widetilde{\nabla f}(x) - \nabla f(x) }^2} \leq \sigma^2.$
\end{enumerate}
\end{assumption}

When an online algorithm incurs small local regret in expectation, it has a convergence guarantee in offline stochastic non-convex optimization:
\begin{proposition}
\label{prop:online-implies-stochastic}
Let $1 \leq w < T$. Suppose that online algorithm $\Acal$ is run on a sequence of $T$ identical loss functions $f(x)$ satisfying Assumption~\ref{assumption:smooth}, with identical stochastic gradient oracles satisfying Assumption~\ref{u1}. Sample $t \sim \mathcal{D}_{[w,T]}$. Then, over the randomness of $t$ and the oracles,
\begin{equation*}
\E \bra{ \norm{ \nabla f(x_t) }^2 }  \leq  \frac{\E\bra{\regret_w(\Acal)}}{T-w}.
\end{equation*}
\end{proposition}
\begin{proof}
Observe that
\[
\mathbb{E}_{t\sim \mathcal{D}_{[w,T]}} \bra{ \norm{\nabla f(x_t)}^2 } \leq \frac{ \sum_{t=1}^T \|\nabla f(x_t)\|^2 }{T-w} \leq \frac{\regret_w(\Acal)}{T-w}.
\]
The claim follows by taking the expectation of both sides, over the randomness of the oracles.
\end{proof}

For a concrete online-to-stochastic reduction, we consider Algorithm~\ref{time-smoothed-ogd2}, which exhibits such a bound on expected local regret.
\begin{algorithm}
\caption{Time-smoothed online gradient descent with stochastic gradient oracles}
\label{time-smoothed-ogd2}
\begin{algorithmic}[1]
\STATE Input: learning rate $\eta > 0$, window size $w \geq 1$.
\STATE Set $x_1=0\in\mathbb{R}^n$ arbitrarily.
\FOR{$t = 1, \ldots, T$}
\STATE Predict $x_t$. Observe the cost function $f_t:\reals^n\to \reals$.
\STATE Update $x_{t+1} := x_t - \frac{\eta}{w} \sum_{i=0}^{w - 1} \widetilde{\nabla f}_{t - i}(x_t)$.
\ENDFOR
\end{algorithmic}
\end{algorithm}

\begin{theorem} 
\label{thm:onco-first2}
Let $f_1, \ldots, f_t$ satisfy Assumption~\ref{assumption:smooth}. Then, Algorithm~\ref{time-smoothed-ogd2}, with access to stochastic gradient oracles $\{\widetilde{\nabla f}_{t}(x_t)\}$ satisfying Assumption~\ref{u1}, and a choice of $\eta = \frac{1}{\beta}$, guarantees
\[\mathbb{E}\left[\regret_w(T)\right] \leq \pa{ 8\beta M + \sigma^2 } \frac{ T }{ w }.\]
Furthermore, Algorithm~\ref{time-smoothed-ogd2} makes a total of $O(Tw)$ calls to the stochastic gradient oracles.
\end{theorem}

Using this expected local regret bound in Proposition~\ref{prop:online-implies-stochastic}, we obtain the reduction claimed at the beginning of the section:

\begin{corollary}
Algorithm~\ref{time-smoothed-ogd2}, with parameter choices $w = \frac{12M\beta + 2\sigma^2}{\eps}$, $T=2w$, and $\eta = \frac{1}{\beta}$, yields
\[ \E\bra{ \norm{ \nabla f(x_t) }^2 } \leq \eps.\]
Furthermore, the algorithm makes $O\pa{ \frac{\sigma^4}{\eps^2} }$ stochastic gradient oracle calls in total.
\end{corollary}
 
\section{An efficient algorithm with second-order guarantees}

We note that by modifying Algorithm~\ref{time-smoothed-ogd} to exploit second-order information, our online algorithm can be improved to play approximate first-order critical points which are also locally almost convex. This entails replacing the gradient descent epochs with a cubic-regularized Newton method \cite{nesterovcubic,Lissa2}.

In this setting, we assume that we have access to each $f_t$ through a value, gradient, and Hessian oracle. That is, once we have observed $f_t$, we can obtain $f_t(x)$, $\nabla f_t(x)$, and $\nabla^2 f_t(x)$ for any $x$. Let $\mathsf{MinEig}(A)$ be the minimum (eigenvalue, eigenvector) pair for matrix $A$. As is standard for offline second-order algorithms, we must add the following additional smoothness restriction:
\begin{assumption}\label{third-order-smooth-assumption}
$f_t$ is twice differentiable and has an $L_2$-Lipschitz Hessian:
\begin{equation*}
\norm{\nabla^2 f(x) - \nabla^2 f(y)} \leq L_2 \norm{x - y}.
\end{equation*}
\end{assumption}

Additionally, we consider only the unconstrained case where $\K = \R^n$; the second-order optimality condition is irrelevant when the gradient does not vanish at the boundary of $\K$.

The second-order Algorithm~\ref{time-smoothed-second-order} uses the same approach as in Algorithm~\ref{time-smoothed-ogd}, but terminates each epoch under a stronger approximate-optimality condition. We define
\[\Phi_t (x) := \max\braces{\norm{\nabla F_{t}(x)}^2, -\frac{4\beta}{3L_2^2} \cdot \Lmin(\nabla^2 F_{t}(x))^3 },\]
so that the quantity $\sum_{t=1}^T \Phi_t(x_t)$ is termwise lower bounded by the costs in $\regret_w (T)$, but penalizes local concavity.

\begin{algorithm}
\caption{Time-smoothed online Newton method}
\label{time-smoothed-second-order}
\begin{algorithmic}[1]
\STATE Input: window size $w \geq 1$, tolerance $\delta > 0$.
\STATE Set $x_1 \in \K$ arbitrarily.
\FOR{$t = 1, \ldots, T$}
\STATE Predict $x_t$. Observe the cost function $f_t:\R^n \to \reals$.
\STATE Initialize $x_{t+1} := x_t$.
\WHILE{$\Phi_{t}(x_{t+1}) > \delta^3 / w^3$}
\STATE Update $x_{t+1} := x_{t+1} - \frac{1}{\beta} \nabla F_{t,w}(x_{t+1})$.
\STATE Let $(\lambda, v) := \mathsf{MinEig}\pa{ \nabla^2 F_{t,w} (x_{t+1}) }$.
\IF{$\lambda < 0$}
\STATE Flip the sign of $v$ so that $\ang{v, \nabla F_{t,w}(x_{t+1})} \leq 0$.
\STATE Compute $y_{t+1} := x_t + \frac{2\lambda}{L_2} v$.
\IF{$F_{t,w}(y_{t+1}) < F_{t,w}(x_{t+1})$}
\STATE Set $x_{t+1} := y_{t+1}$.
\ENDIF
\ENDIF
\ENDWHILE
\ENDFOR
\end{algorithmic}
\end{algorithm}

We characterize the convergence and oracle complexity properties of this algorithm:
\begin{theorem}
Let $f_1, \ldots, f_T$ be the sequence of loss functions presented to Algorithm~\ref{time-smoothed-second-order}, satisfying Assumptions~\ref{assumption:smooth} and \ref{third-order-smooth-assumption}. Choose $\delta = \beta$. Then, for some constants $C_1, C_2$ in terms of $M, L, \beta, L_2$:
\begin{enumerate}
\item[(i)] The iterates $\{x_t\}$ produced by Algorithm~\ref{time-smoothed-second-order} satisfy
\[\sum_{t=1}^T \Phi_t(x_t) \leq C_1 \cdot \frac{T}{w^2}.\]
\item[(ii)] The total number of iterations $\tau$ of the inner loop taken by Algorithm~\ref{time-smoothed-second-order} satisfies
\[\tau \leq C_2 \cdot Tw^2.\]
\end{enumerate}
\end{theorem}
\begin{proof}[Proof of (i).]
For each $1 \leq t \leq T$, we have
\[\Phi_{t-1}(x_t) \leq \frac{\delta^3}{w^3}.\]
Let $h_t(x) := \frac{1}{w} \pa{ f_t(x) - f_{t-w}(x) }$.
Then, since $h_t(x)$ is $\frac{2L}{w}$-Lipschitz and $\frac{2\beta}{w}$-smooth,
\begin{align*}
\Phi_t(x_t) &= \max\Big\{ \norm{\nabla F_{t-1}(x_t) + \nabla h_t(x_t)}^2, \\
& \;\;\;\;\;\; -\frac{4\beta}{3L_2^2} \cdot \Lmin(\nabla^2 F_{t}(x_t) + \nabla^2 h_t(x_t))^3 \Big\} \\
&\leq \max\braces{\pa{\frac{\delta^{3/2}}{w^{3/2}} + \frac{2L}{w} }^2, \; \frac{4\beta}{3L_2^2} \cdot \pa{ \frac{\delta}{w} + \frac{2\beta}{w} }^3},
\end{align*}
which is bounded by $C_1/w^2$, for some $C_1(M, L, \beta, L_2^2)$. The claim follows by summing this inequality across all $1 \leq t \leq T$.
\end{proof}
\begin{proof}[Proof of (ii).] We first show the following progress lemma:
\begin{lemma}
\label{lem:second-order-progress}
Let $z, z'$ be two consecutive iterates of the inner loop in Algorithm~\ref{time-smoothed-second-order} during round $t$. Then,
\[F_t(z') - F_t(z) \leq -\frac{\Phi_t(z)}{2\beta}.\]
\end{lemma}
\begin{proof}
Let $u$ denote the step $z' - z$. Let $g := \nabla F_t(z)$, $H := \nabla^2 F_t(z)$, and $(\lambda, v) := \mathsf{MinEig}(H)$.

Suppose that at time $t$, the algorithm takes a gradient step, so that $u = g/\beta$. Then, by second-order smoothness of $F_t$, we have
\[F_t(z') - F_t(z) \leq \ang{g, u} + \frac{\beta}{2}\norm{u}^2 = -\frac{1}{2\beta} \norm{g}^2.\]
Supposing instead that the algorithm takes a second-order step, so that $u = \pm \frac{2\lambda}{L_2} v$ (whichever sign makes $\ang{g,u} \leq 0$), the third-order smoothness of $F_t$ implies
\begin{align*}
F_{t}(z') - F_{t}(z) &\leq \ang{ g, u } + \frac{1}{2} u^T H_t u + \frac{L_2}{6}\norm{u}^3 \\
&= \ang{ g, u } + \frac{\lambda}{2} \norm{u}^2 + \frac{L_2}{6}\norm{u}^3 \\
&\leq \frac{2\lambda^3}{3L_2^2} = \frac{1}{2\beta} \cdot \frac{4\beta\lambda^3}{3L_2^2}.
\end{align*}
The lemma follows due to the fact that the algorithm takes the step that gives a smaller value of $F_t(z')$.
\end{proof}
Following the technique from Theorem~\ref{thm:onco-first}, for $2 \leq t \leq T$, let $\tau_t$ be the number of iterations of the inner loop during the execution of Algorithm 3 during round $t-1$ (in order to generate the iterate $x_t$). Then, we have the following lemma:
\begin{lemma}
\label{lem:second-order-tau}
For any $2 \leq t \leq T$,
\[F_{t-1}(x_t) - F_{t-1}(x_{t-1}) \leq -\tau_t \cdot \frac{\delta^3}{2\beta w^3}. \]
\end{lemma}
\begin{proof}
This follows by summing the inequality Lemma 5.3 for across all pairs of consecutive iterates of the inner loop within the same epoch, and noting that each term $\Phi(z)$ is at least $\frac{\delta^3}{w^3}$ before the inner loop has terminated.
\end{proof}
Finally, we write (understanding $F_0(x_0) := 0$):
\begin{align*}
&F_T(x_T) = \sum_{t=1}^{T} F_{t}(x_{t}) - F_{t-1}(x_{t-1}) \\
&= \sum_{t=1}^{T} F_{t-1}(x_t) - F_{t-1}(x_{t-1}) + f_t(x_t) - f_{t-w}(x_t) \\
&\leq \sum_{t=2}^{T} \bra{ F_{t-1}(x_t) - F_{t-1}(x_{t-1}) } + \frac{2MT}{w}.
\end{align*}
Using Lemma~\ref{lem:second-order-tau}, we have
\begin{align*}
F_T(x_T) &\leq \frac{2MT}{w} - \frac{\delta^3}{2\beta w^3} \cdot \sum_{t=1}^T \tau_t,
\end{align*}
whence
\begin{align*}
\tau = \sum_{t=1}^T \tau_t &\leq \frac{2\beta w^3}{\delta^3} \cdot \pa{ \frac{2MT}{w} - F_T(x_T) } \\
&\leq \frac{2\beta M}{\delta^3} \cdot \pa{2Tw^2 + w^3} \\
&\leq \frac{6 M}{\beta^2} \cdot Tw^2,
\end{align*}
as claimed (recalling that we chose $\delta = \beta$ for this analysis).
\qed
\end{proof}
 
\section{A solution concept for non-convex games}

Finally, we discuss an application of our regret minimization framework to learning in $k$-player $T$-round iterated games with smooth, non-convex payoff functions. Suppose that each player $i \in [k]$ has a fixed decision set $\K_i \subset \R^n$, and a fixed payoff function $f_i : \K \rightarrow \R$ satisfies Assumption~\ref{assumption:smooth} as before. Here, $\K$ denotes the Cartesian product of the decision sets $\K_i$: each payoff function is defined in terms of the choices made by every player.

In such a game, it is natural to consider the setting where players will only consider small \emph{local} deviations from their strategies. This is a natural setting, which models risk aversion. This setting lends itself to the notion of a local equilibrium, to replace the stronger condition of Nash equilibrium: a joint strategy in which no player encounters a large gradient on her utility. However, finding an approximate local equilibrium in this sense remains computationally intractable when the utility functions are non-convex.

Using the idea of time-smoothing, we formulate a tractable relaxed notion of local equilibrium, defined over some time window $w$. Intuitively, this definition captures a state of an iterated game in which each player examines the past $w$ actions played, and no player can make small deviations to improve the average performance of her play against her opponents' historical play. We formulate this solution concept as follows:

\begin{definition}[Smoothed local equilibrium]
Fix some $\eta > 0, w \geq 1$. Let $\braces{ f_i(x^1, \ldots, x^k) : \K \rightarrow \R }_{i=1}^k$ be the payoff functions for a $k$-player iterated game. A joint strategy $(x_t^1, \ldots, x_t^k)$ is an $\eps$-approximate $(\eta, w)$-\emph{smoothed local equilibrium} with respect to past iterates $\braces{ (x_{t-j}^1, \ldots, x_{t-j}^k) }_{j=0}^{w-1}$ if, for every player $i \in [k]$,
\begin{equation*}
\left\lVert \nabla_{\K, \eta} \bra{\frac{ \sum_{j=0}^{w-1} \tilde f_{i,t-j} }{w} }(x_t^i) \right\rVert \leq \eps,
\end{equation*}
where
\begin{equation*}
\tilde f_{i,t'}(x) \equaldef f_i(x_{t'}^1, \ldots, x_{t'}^{i-1}, x, x_{t'}^{i+1}, \ldots, x_{t'}^k).
\end{equation*}
\end{definition}

To achieve such an equilibrium efficiently, we use Algorithm~\ref{game-player}, which runs $k$ copies of any online algorithm that achieves a $w$-local regret bound for some $\eta > 0$.
\begin{algorithm}
\caption{Time-smoothed game simulation}
\label{game-player}
\begin{algorithmic}[1]
\STATE Input: convex decision sets $\K_1, \ldots, \K_k \subseteq \reals^n$, payoff functions $f_i : (\K_1, \ldots, \K_k) \rightarrow \R$, online algorithm $\Acal$, window size $1 \leq w < T$.
\STATE Initialize $k$ copies $(\Acal_1, \ldots, \Acal_k)$ of $\Acal$ with window size $w$, where each $\Acal_i$ plays on decision set $\K_i$.
\FOR{$t = 1, \ldots, T$}
\STATE Each $\Acal_i$ outputs $x_t^i$.
\STATE Show each $\Acal_i$ the online loss function
\[f_{i,t}(x) := -f_i(x^1_t, \ldots, x^{i-1}_t, x, x^{i+1}_{t}, \ldots, x^k_t).\]
\ENDFOR
\end{algorithmic}
\end{algorithm}

We show this meta-algorithm yields a subsequence of iterates that satisfy our solution concept, with error parameter dependent on the local regret guarantees of each player:
\begin{theorem}
For some $t$ such that $w \leq t \leq T$, the joint strategy $(x_t^1, \ldots, x_t^k)$ produced by Algorithm~\ref{game-player} is an $\eps$-approximate ($\eta$, $w$)-smoothed local equilibrium with respect to $\braces{ (x_{t-j}^1, \ldots, x_{t-j}^k) }_{j=0}^{t-1}$, where
\[\eps = \sqrt{ \sum_{i=1}^k \frac{\regret_{w,\Acal_i}(T)}{T - w} }.\]
\end{theorem}
\begin{proof}
Summing up the definitions of $w$-regret bounds achieved by each $\Acal$, and truncating the first $w-1$ terms, we get
\begin{equation*}
\sum_{i=1}^k \sum_{t=w}^T \norm{ \nabla_{\K, \eta} F_t^i(x_t^i) }^2 \leq \sum_{i=1}^k \regret_{w,\Acal_i}(T).
\end{equation*}
Thus, for some $t$ between $w$ and $T$ inclusive, it holds that
\begin{align*}
\sum_{i=1}^k \left\lVert \nabla_{\K, \eta} \bra{ \frac{ \sum_{j=0}^{w-1} \tilde f_{i,t-j} }{w}}(x_t^i) \right\rVert ^2
&=
\sum_{i=1}^k \norm{ \nabla_{\K, \eta} F_t^i(x_t^i) }^2 \\
\leq
\sum_{i=1}^k \frac{\regret_{w,\Acal_i}(T)}{T - w}.
\end{align*}
Thus, for the same $t$ we have
\begin{equation*}
\max_{i\in [k]} \left\lVert \nabla_{\K, \eta} \bra{ \frac{ \sum_{j=0}^{w-1} \tilde f_{i,t-j} }{w}}(x_t^i) \right\rVert
\leq
\sqrt{\sum_{i=1}^k \frac{\regret_{w,\Acal_i}(T)}{T - w}},
\end{equation*}
as claimed.
\end{proof}

\subsection{Experience replay for GAN training}
The training of generative adversarial networks (GANs), a popular generative model, can be viewed as a symmetric game with a non-convex payoff function. In this section, we apply and contextualize our framework of smoothed local equilibrium for GANs.

In the seminal setting of \cite{goodfellow2014generative}, there are two players: a generator who wants to imitate samples from a ``true'' distribution $\mathcal{D}$ on $\R^n$, and a discriminator who wants to distinguish true samples from $\D$ and fake samples produced by the generator. The generator chooses some function $G : \R^m \rightarrow \R^n$, which maps input randomness $z \sim \mathcal{D}'$ to fake samples. The discriminator chooses a function $D: \R^n \rightarrow [0,1]$, a guess for the likelihood that a data point is real. $D$ and $G$ are chosen from some function classes parameterized by $\theta_D$ and $\theta_G$, often both neural networks. The generator and discriminator play an iterated game with the objective function
\[ L(\theta_D, \theta_G) := \E_{x \sim \mathcal{D}, z \sim \mathcal{D}'} \bra{ \log D(x) + \log \pa{1 - D(G(z))} }. \]
In the language of our model, the discriminator's payoff is $L(\theta_D, \theta_G)$, while the generator's payoff is $-L(\theta_D, \theta_G)$, and the players access $L$ via stochastic gradient oracles, as in Section~\ref{stochastic-section}. Indeed, in GAN training, it is a standard technique to update $\theta_D$ and $\theta_G$ incrementally in lockstep, via stochastic gradient descent steps. This is very similar to using Algorithm~\ref{time-smoothed-ogd2} as the local regret minimization algorithm driving Algorithm~\ref{game-player} with $w = 1$; the only difference is that the players' updates are alternating rather than simultaneous.

Instability is a major challenge for GAN training, and improving training stability is a highly active research area in deep learning. To this end, our local regret framework provides a meaningful yet attainable theoretical goal, which is met by our time-smoothed gradient-based algorithms. For this game, a smoothed local equilibrium with window parameter $w$ guarantees that the generator and discriminator simultaneously encounter small averaged gradients on each other's past $w$ choices of functions. This is a particularly appealing notion of equilibrium in the setting of GAN training, as it implies that a gradient-based training process becomes approximately stationary.

Indeed, maintaining a buffer of past discriminators (running Algorithm~\ref{time-smoothed-ogd2}) is a known technique for stabilizing GAN training. \cite{metz2016unrolled} In reinforcement learning, this corresponds to a form of experience replay. \cite{pfau2016connecting}

\section{Concluding remarks}
We have described how to extend the theory of online learning to non-convex loss functions, while permitting \emph{efficient} algorithms. Our definitions give rise to efficient online and stochastic non-convex optimization algorithms that converge to local optima of first and second order. We give a game-theoretic solution concept which we call local equilibrium, which, in contrast to existing solution concepts such as Nash equilibrium, is efficiently attainable in any non-convex game.

\section*{Acknowledgments}
We thank Naman Agarwal, Brian Bullins, Matt Weinberg, and Yi Zhang for helpful discussions.

\bibliographystyle{alpha}
\bibliography{main-arxiv-expanded}

\newcommand{\etalchar}[1]{$^{#1}$}
\begin{thebibliography}{GPAM{\etalchar{+}}14}

\bibitem[AAZB{\etalchar{+}}16]{Lissa2}
Naman Agarwal, Zeyuan Allen-Zhu, Brian Bullins, Elad Hazan, and Tengyu Ma.
\newblock Finding approximate local minima for nonconvex optimization in linear
  time.
\newblock {\em arXiv preprint arXiv:1611.01146}, 2016.

\bibitem[ABH16]{LiSSA2016}
Naman Agarwal, Brian Bullins, and Elad Hazan.
\newblock Second order stochastic optimization for machine learning in linear
  time.
\newblock {\em arXiv preprint arXiv:1602.03943}, 2016.

\bibitem[AHK12]{AHK-MW}
Sanjeev Arora, Elad Hazan, and Satyen Kale.
\newblock The multiplicative weights update method: a meta-algorithm and
  applications.
\newblock {\em Theory of Computing}, 8(6):121--164, 2012.

\bibitem[AZH16]{allen2016variance}
Zeyuan Allen-Zhu and Elad Hazan.
\newblock Variance reduction for faster non-convex optimization.
\newblock In {\em Proceedings of The 33rd International Conference on Machine
  Learning}, pages 699--707, 2016.

\bibitem[BM05]{blumMansour}
A.~Blum and Y.~Mansour.
\newblock From external to internal regret.
\newblock In {\em COLT}, pages 621--636, 2005.

\bibitem[CBL06]{CesaBianchiLugosi06book}
Nicol{\`o} Cesa-Bianchi and G\'abor Lugosi.
\newblock {\em Prediction, Learning, and Games}.
\newblock Cambridge University Press, 2006.

\bibitem[CDHS16]{CarmonAGD}
Yair Carmon, John~C. Duchi, Oliver Hinder, and Aaron Sidford.
\newblock Accelerated methods for non-convex optimization.
\newblock {\em arXiv preprint 1611.00756}, 2016.

\bibitem[Cov91]{cover}
Thomas Cover.
\newblock Universal portfolios.
\newblock {\em Math. Finance}, 1(1):1--19, 1991.

\bibitem[DHS11]{adagrad}
John Duchi, Elad Hazan, and Yoram Singer.
\newblock Adaptive subgradient methods for online learning and stochastic
  optimization.
\newblock {\em The Journal of Machine Learning Research}, 12:2121--2159, 2011.

\bibitem[EM15]{newsamp}
Murat~A Erdogdu and Andrea Montanari.
\newblock Convergence rates of sub-sampled newton methods.
\newblock In {\em Advances in Neural Information Processing Systems}, pages
  3034--3042, 2015.

\bibitem[FS97]{FreundSch1997}
Yoav Freund and Robert~E. Schapire.
\newblock A decision-theoretic generalization of on-line learning and an
  application to boosting.
\newblock {\em J. Comput. Syst. Sci.}, 55(1):119--139, August 1997.

\bibitem[GL13]{ghadimi-lan}
Saeed Ghadimi and Guanghui Lan.
\newblock Stochastic first-and zeroth-order methods for nonconvex stochastic
  programming.
\newblock {\em SIAM Journal on Optimization}, 23(4):2341--2368, 2013.

\bibitem[GPAM{\etalchar{+}}14]{goodfellow2014generative}
Ian Goodfellow, Jean Pouget-Abadie, Mehdi Mirza, Bing Xu, David Warde-Farley,
  Sherjil Ozair, Aaron Courville, and Yoshua Bengio.
\newblock Generative adversarial nets.
\newblock In {\em Advances in neural information processing systems}, pages
  2672--2680, 2014.

\bibitem[Haz16]{OCObook}
Elad Hazan.
\newblock Introduction to online convex optimization.
\newblock {\em Foundations and TrendsÂ® in Optimization}, 2(3-4):157--325,
  2016.

\bibitem[HK07]{NIPS2007_695}
Elad Hazan and Satyen Kale.
\newblock Computational equivalence of fixed points and no regret algorithms,
  and convergence to equilibria.
\newblock In J.c. Platt, D.~Koller, Y.~Singer, and S.~Roweis, editors, {\em
  Advances in Neural Information Processing Systems 20}, pages 625--632. MIT
  Press, Cambridge, MA, 2007.

\bibitem[HMC00]{hart2000simple}
Sergiu Hart and Andreu Mas-Colell.
\newblock A simple adaptive procedure leading to correlated equilibrium.
\newblock {\em Econometrica}, 68(5):1127--1150, 2000.

\bibitem[MPPSD16]{metz2016unrolled}
Luke Metz, Ben Poole, David Pfau, and Jascha Sohl-Dickstein.
\newblock Unrolled generative adversarial networks.
\newblock {\em arXiv preprint arXiv:1611.02163}, 2016.

\bibitem[Nes04]{NesterovBook}
Yurii Nesterov.
\newblock {\em Introductory lectures on convex optimization}, volume~87.
\newblock Springer Science \& Business Media, 2004.

\bibitem[NP06]{nesterovcubic}
Yurii Nesterov and Boris~T Polyak.
\newblock Cubic regularization of newton method and its global performance.
\newblock {\em Mathematical Programming}, 108(1):177--205, 2006.

\bibitem[PV16]{pfau2016connecting}
David Pfau and Oriol Vinyals.
\newblock Connecting generative adversarial networks and actor-critic methods.
\newblock {\em arXiv preprint arXiv:1610.01945}, 2016.

\bibitem[SS11]{shalev2011online}
Shai Shalev-Shwartz.
\newblock Online learning and online convex optimization.
\newblock {\em Foundations and Trends in Machine Learning}, 4(2):107--194,
  2011.

\bibitem[Vov90]{Vovk:1990}
Volodimir~G. Vovk.
\newblock Aggregating strategies.
\newblock In {\em Proceedings of the Third Annual Workshop on Computational
  Learning Theory}, COLT '90, pages 371--386, 1990.

\end{thebibliography}

\begin{appendix}
\section{Proof of Theorem 4.4}

Since each $f_t$ is $\beta$-smooth, it follows that each $F_t$ is $\beta$-smooth. Define $\widehat{\nabla f_{t}} = \frac{x_t-x_{t+1}}{\eta}$. 
 Note that since the iterates $(x_t:t\in [T])$ depend on the gradient estimates, the iterates are stochastic variables, as are $\widehat{\nabla f_{t}}$. By $\beta$-smoothness of $F_t$, we have
\begin{align*}
&F_{t,w}(x_{t+1})-F_{t,w}(x_t) \\
\leq& \ang{\nabla F_{t,w}(x_t), x_{t+1}-x_t}+\frac{\beta}{2} \|x_{t+1}-x_t\|^2\\
 =& -\eta\ang{\nabla F_{t,w}(x_t), \widehat{\nabla f_{t}}} + \eta^2\frac{\beta}{2} \|\widehat{\nabla f_{t}}\|^2 \\
 =& -\eta \|\nabla F_{t,w}(x_t)\|^2 - \eta\ang{\nabla F_{t,w}(x_t), \widehat{\nabla f_{t}}-\nabla F_{t,w}(x_t)} \\
 &+ \eta^2\frac{\beta}{2}\left( \|\nabla F_{t,w}(x_t)\|^2 \right)\\
 &+ \eta^2\frac{\beta}{2}\left(2\ang{\nabla F_{t,w}(x_t),\widehat{\nabla f_t}-\nabla F_{t,w}(x_t)}\right)\\
 &+  \eta^2\frac{\beta}{2}\left(\|\widehat{\nabla f_{t}}-\nabla F_{t,w}(x_t)\|^2\right) \\
 =& - \left(\eta-\frac{\beta}{2}\eta^2\right) \|\nabla F_{t,w}(x_t)\|^2 \\
 &- (\eta-\beta \eta^2)\ang{\nabla F_{t,w}(x_t), \widehat{\nabla f_{t}}-\nabla F_{t,w}(x_t)} \\
 &+ \eta^2\frac{\beta}{2}  \|\widehat{\nabla f_{t}}-\nabla f(x_t)\|^2.
 \end{align*}

 Additionally, we each observe that $\widehat{\nabla f_t}$ is an average of $w$ independently sampled unbiased gradient estimates of variance $\sigma^2$ each. It follows as a consequence that
 \begin{align*}
 & \mathbb{E} \big[ \widehat{\nabla f_{t}} \big| x_t\big] = \nabla F_{t,w}(x_t)\\
 & \mathbb{E} \big[ \|\widehat{\nabla f_{t}} - \nabla F_{t,w}(x_t) \|^2 \big| x_t \big] \leq \frac{\sigma^2}{w}
 \end{align*}
 Now, applying $\mathbb{E}\left[\cdot|x_t\right]$ on both sides, it follows that
 \begin{align*}
\left(\eta-\frac{\beta}{2}\eta^2\right) &\cdot \mathbb{E}\|\nabla F_{t,w}(x_t)\|^2 \\
 &\leq \mathbb{E}\left[F_{t,w}(x_t)-F_{t,w}(x_{t+1}) \right]+ \eta^2\frac{\beta}{2}  \frac{\sigma^2}{w}.
 \end{align*}
 Also, we note that
\begin{align*}
&F_{t+1,w}(x_{t+1}) - F_{t,w}(x_{t+1})\\
 =& \frac{1}{w}\sum_{i=0}^{w-1} f_{t+1-i}(x_{t+1}) - \frac{1}{w}\sum_{i=0}^{w-1} f_{t-i}(x_{t+1}) \\
=& \frac{1}{w}\sum_{i=-1}^{w-2} f_{t-i}(x_{t+1}) - \frac{1}{w}\sum_{i=0}^{w-1} f_{t-i}(x_{t+1}) \\
=& \frac{f_{t+1}(x_{t+1}) - f_{t-w+1}(x_{t+1})}{w} \leq \frac{2M}{w} 
\end{align*}
Adding the last two inequalities, we proceed to sum the above inequality over all time steps:
 \begin{align*}
 \mathbb{E} \left[\sum_{t=1}^T \|\nabla F_{t,w}(x_t)\|^2 \right] \leq \frac{2M + \frac{2MT}{w}+ \frac{T\beta \eta^2}{2w} \sigma^2}{\eta - \frac{\beta \eta^2}{2}}.
 \end{align*}

Setting $\eta = 1/\beta$ yields the claim from the theorem.

Finally, note that for each round the number of stochastic gradient oracle calls required is $w$. Therefore, across all $T$ rounds, the number of noisy oracle calls is $Tw$. \qed 
\section{Proof of Theorem 5.1 (ii)}
Following the technique from Theorem 3.1, for $2 \leq t \leq T$, let $\tau_t$ be the number of iterations of the inner loop during the execution of Algorithm 3 during round $t-1$ (in order to generate the iterate $x_t$). Then, we have the following lemma:
\begin{lemma}
\label{lem:second-order-tau}
For any $2 \leq t \leq T$,
\[F_{t-1}(x_t) - F_{t-1}(x_{t-1}) \leq -\tau_t \cdot \frac{\delta^3}{2\beta w^3}. \]
\end{lemma}
\begin{proof}
This follows by summing the inequality Lemma 5.3 for across all pairs of consecutive iterates of the inner loop within the same epoch, and noting that each term $\Phi(z)$ is at least $\frac{\delta^3}{w^3}$ before the inner loop has terminated.
\end{proof}
Finally, we write (understanding $F_0(x_0) := 0$):
\begin{align*}
&F_T(x_T) = \sum_{t=1}^{T} F_{t}(x_{t}) - F_{t-1}(x_{t-1}) \\
&= \sum_{t=1}^{T} F_{t-1}(x_t) - F_{t-1}(x_{t-1}) + f_t(x_t) - f_{t-w}(x_t) \\
&\leq \sum_{t=2}^{T} \bra{ F_{t-1}(x_t) - F_{t-1}(x_{t-1}) } + \frac{2MT}{w}.
\end{align*}
Using Lemma~\ref{lem:second-order-tau}, we have
\begin{align*}
F_T(x_T) &\leq \frac{2MT}{w} - \frac{\delta^3}{2\beta w^3} \cdot \sum_{t=1}^T \tau_t,
\end{align*}
whence
\begin{align*}
\tau = \sum_{t=1}^T \tau_t &\leq \frac{2\beta w^3}{\delta^3} \cdot \pa{ \frac{2MT}{w} - F_T(x_T) } \\
&\leq \frac{2\beta M}{\delta^3} \cdot \pa{2Tw^2 + w^3} \\
&\leq \frac{6 M}{\beta^2} \cdot Tw^2,
\end{align*}
as claimed (recalling that we chose $\delta = \beta$ for this analysis).
\qed 
\section{Proof of Theorem 6.2}

Summing up the definitions of $w$-regret bounds achieved by each $\Acal$, and truncating the first $w-1$ terms, we get
\begin{equation*}
\sum_{i=1}^k \sum_{t=w}^T \norm{ \nabla_{\K, \eta} F_t^i(x_t^i) }^2 \leq \sum_{i=1}^k \regret_{w,\Acal_i}(T).
\end{equation*}
Thus, for some $t$ between $w$ and $T$ inclusive, it holds that
\begin{align*}
\sum_{i=1}^k \left\lVert \nabla_{\K, \eta} \bra{ \frac{ \sum_{j=0}^{w-1} \tilde f_{i,t-j} }{w}}(x_t^i) \right\rVert ^2
&=
\sum_{i=1}^k \norm{ \nabla_{\K, \eta} F_t^i(x_t^i) }^2 \\
\leq
\sum_{i=1}^k \frac{\regret_{w,\Acal_i}(T)}{T - w}.
\end{align*}
Thus, for the same $t$ we have
\begin{equation*}
\max_{i\in [k]} \left\lVert \nabla_{\K, \eta} \bra{ \frac{ \sum_{j=0}^{w-1} \tilde f_{i,t-j} }{w}}(x_t^i) \right\rVert
\leq
\sqrt{\sum_{i=1}^k \frac{\regret_{w,\Acal_i}(T)}{T - w}},
\end{equation*}
as claimed.
\qed
 
\end{appendix}

\end{document}